%% file: main.tex
\documentclass[wcp]{jmlr}

\usepackage{longtable}

\usepackage{booktabs}

\usepackage{lineno}
\linenumbers

\makeatletter
\let\Ginclude@graphics\@org@Ginclude@graphics 
\makeatother

\jmlrvolume{}
\jmlryear{2025}
\jmlrworkshop{ACML 2025}

\usepackage[disable]{todonotes}
\newcommand{\oam}[1]{\todo[inline,color=orange!40]{{\textbf{OM:}~}#1}}
\newcommand{\ap}[1]{\todo[inline,color=red!40]{{\textbf{AP:}~}#1}}

\newcommand{\KL}{\operatorname{KL}\!}

\title[\textsc{KL-UCB-Transfer}]{Asymptotically Optimal Problem-Dependent Bandit Policies for Transfer Learning}

\author{\Name{Adrien PREVOST}  \Email{adrien.prevost@inria.fr}\\ 
\Name{Timothee MATHIEU} \Email{timothee.mathieu@inria.fr} \\
\Name{Odalric-Ambrym MAILLARD} \Email{odalric.maillard@inria.fr} \\
\addr Equipe Scool, Univ. Lille, Inria, CNRS, Centrale Lille, UMR 9189- CRIStAL, F-59000 Lille, France}

\begin{document}
\nolinenumbers

\maketitle

\begin{abstract}
We study the non‐contextual multi‐armed bandit problem in a transfer learning setting: before any pulls, the learner is given $N'_k$ i.i.d.\ samples from each source distribution $\nu'_k$, and the true target distributions $\nu_k$ lie within a known distance bound $d_k(\nu_k,\nu'_k)\le L_k$.  In this framework, we first derive a problem‐dependent asymptotic lower bound on cumulative regret that extends the classical Lai–Robbins result to incorporate the transfer parameters $(d_k,L_k,N'_k)$.  We then propose \textsc{KL-UCB-Transfer}, a simple index policy that matches this new bound in the Gaussian case.  Finally, we validate our approach via simulations, showing that \textsc{KL-UCB-Transfer} significantly outperforms the no‐prior baseline when source and target distributions are sufficiently close.
\end{abstract}


\section{Introduction}

In many sequential‐decision tasks, one often has access to \emph{historical data} collected under related but not identical conditions.  Leveraging such data to speed up learning in a new environment is the goal of \emph{transfer learning}, a paradigm that has been successfully applied in computer vision, reinforcement learning, and natural language processing.  For example, pre‐trained image classifiers built on large, generic datasets can be adapted to medical imaging; in robotics, policies learned in simulation help bootstrap real‐world control; and in agriculture, past yield records from neighboring fields guide crop‐management decisions in a new plot \citep{taylor2009transfer}.

In this work we study transfer in the simplest possible online setting: a \emph{non‐contextual multi‐armed bandit} with $K$ arms and i.i.d.\ rewards.  Before interacting with the target bandit, the learner receives $N'_k$ i.i.d.\ samples from each \emph{source} arm $k$, whose true reward distribution we denote by $\nu'_k$.  The goal is to minimize cumulative regret on the \emph{target} bandit, whose arm $k$ has distribution $\nu_k$.  We assume the learner knows a family of (pseudo-)distance $d_k$ and upper bounds $L_k > 0$ such that each $\nu_k$ lies within distance $L_k$ of $\nu'_k$:
\begin{equation*}
    \nu_k \in \bigl\{\gamma \in \mathcal{D} : d_k(\gamma,\nu'_k)\le L_k\bigr\},
\quad
k=1,\dots,K.
\end{equation*}

Intuitively, small $L_k$ makes source samples informative; large $N'_k$ reduces noise. This departs from prior transfer-bandit work assuming contextual, causal, or multi-task structure \cite{ rahul2024transfer, wang2022thompson}. Instead, we study a minimal non-contextual setting with i.i.d.\ samples per arm and transfer via known distance bounds $(d_k,L_k)$. This captures scenarios (e.g., warm-start from historical data or sim-to-real) while enabling sharp analysis. This \emph{non‐contextual transfer bandit} setting raises two key questions:
\begin{enumerate}
  \item \textbf{Characterize the transfer-aware minimal regret.}  
    Extend the classical Lai–Robbins lower bound on the regret \cite{LAI19854} so it explicitly depends on the transfer parameters $(d_k,L_k,N'_k)$, yielding the minimal asymptotic cost any algorithm must pay.
  \item \textbf{Design a matching policy.}  
    Propose a simple index strategy that incorporates the $N'_k$ prior samples and achieves the transfer-aware lower bound.
\end{enumerate}

\noindent\textbf{Our contributions}  

\begin{itemize}
  \item We derive a \emph{problem‐dependent} lower bound on cumulative regret which holds for any algorithm with prior data $\{N'_k,\hat\nu'_k,d_k,L_k\}$ (see Section \ref{sec:LowerBounds}, Theorem \ref{thm:lowerb}).  
  \item Focusing on Gaussian rewards with known variance, we introduce in Section \ref{sec:KLUCB} \textsc{KL-UCB-Transfer}, a minor modification of the KL‐UCB index that adds a KL‐penalty term for the source data.  We prove that the regret of \textsc{KL-UCB-Transfer} \emph{matches} our lower bound (see Section \ref{sec:RegretKL}, Theorem \ref{thm:KLREGRET}). In the regret analysis, one must carefully handle potentially large or $T$-dependent $N'_k$: the index should leverage prior samples to avoid redundant exploration while preserving sufficient exploration when necessary.
  \item We validate via simulations that \textsc{KL-UCB-Transfer} can achieve substantial regret reductions when the priors on suboptimal arms are accurate, and we highlight the delicate short‐term trade‐offs when priors are placed on the optimal arm in Section \ref{sec:sim}.
\end{itemize}

By keeping our setting minimal—no contexts, fixed variances—we obtain a clean, intuitive theory and a useful algorithm for studying transfer learning.

\section{Related Work}

\paragraph{Classical non‐contextual bandits.}  
The stochastic multi‐armed bandit was first studied by \cite{dc35850b-2ca1-314f-9e0d-470713436b17} and rigorously analyzed by \cite{robbins1952some}. \cite{LAI19854} proved the fundamental problem‐dependent asymptotic lower bound on regret, and a range of algorithms—such as KL‐UCB~\cite{pmlr-v19-maillard11a,Cappe2013KL}, IMED~\cite{honda2015imed}, DMED~\cite{honda2010asymptotically}, and Thompson Sampling~\cite{kaufmann2012thompson}—are known to match this bound.

\paragraph{Contextual transfer and causal approaches.}  
Transfer in contextual bandits has been extensively studied: \cite{deng2025transferlearninglatentcontextual} leverages causal transportability to handle covariate shift in latent contextual settings; \cite{10.1214/23-AOS2341} derives minimax optimal rates under covariate shift; \cite{NEURIPS2023_8a8ce53b} proposes a transport-aware Thompson Sampling policy using graphical causal models; \cite{ijcai2017p186} analyzes a causal MAB with latent confounders and provides a problem-dependent lower bound under specific structural assumptions. However, these approaches rely on contextual or causal-structure assumptions and do not address the non‐contextual offline–online transfer setting where one first observes i.i.d.\ source samples per arm, nor do they derive a problem-dependent lower bound, and consequently lack any algorithm matching such a bound in that regime.

\paragraph{Non‐contextual transfer bandits.}  
Several recent works consider transfer in the classical (non‐contextual) bandit setting, but under different task structures \cite{NIPS2013_062ddb6c,10.1007/978-3-030-86486-6_1}. \cite{wang2022thompson} study a \emph{multi‐task} regime where multiple bandits run in parallel and Thompson Sampling borrows posterior mass across tasks for robustness, yielding minimax‐style guarantees.  \cite{rahul2024transfer,rahul2024exploiting} analyze a \emph{sequential‐task} framework: they assume known “adjacent similarity” bounds to transfer raw reward samples via UCB, and in their follow‐up work they learn the similarity radius online, each time achieving improved instance‐dependent regret over no‐transfer baselines. \oam{The was this is presented, this looks very close to the current setting: why is this work not enough? explain.} 
\ap{Le paragraphe suivant parle justement des limites des papiers présentés}
\oam{Implicitement peut-être, pas explcitement.}
\\
\\
By contrast, we consider a single target bandit and first receive $N'_k$ i.i.d. samples from each source arm.  We give the first \emph{problem‐dependent asymptotic lower bound} in this \emph{offline, non‐contextual} transfer setting (see Theorem \ref{thm:lowerb}) and design a KL‐UCB–style index that \emph{matches} it exactly in the Gaussian case (see Theorem \ref{thm:KLREGRET}). 

While our setting may appear overly simplified, deriving a problem-dependent lower bound that explicitly accounts for arbitrary prior sample sizes $N'_k$ is non-trivial and not covered by existing approaches, which often merge prior and online data without regard to lower-bound tightness.
Also, despite the generic formulation, this setting has received little direct attention in the literature.

\section{Problem formulation}

In this section, we first recall the classical stochastic multi‐armed bandit framework, then introduce our offline non‐contextual transfer setting with prior samples from a related source.

\paragraph{Classical setting}  
We consider a known family of distributions $\mathcal{D}$ and unknown arm distributions $\nu=(\nu_1,\dots,\nu_K)\in\mathcal{D}^K$ with means $\mu_k$. At each round $t$, a learner selects arm $a_t$ and observes reward $X_t\sim\nu_{a_t}$. Assume w.l.o.g.\ that arm 1 has the highest mean. Let $\mathbf{1}\{\cdot\}$ denote the indicator. the cumulative regret is
\begin{equation*}
  R_T = \sum_{k=1}^K \mathbb{E}[N_k(T)]\,(\mu_1-\mu_k),
  \quad \text{where} \hspace{0.5cm} N_k(T)=\sum_{t=1}^T\mathbf{1}\{a_t=k\}.
\end{equation*}
A strategy is \emph{consistent} if for each suboptimal arm $k$, $\mathbb{E}[N_k(T)]=o(T^\alpha)$ for all $\alpha>0$. For any consistent algorithm and suboptimal arm $k$, the following asymptotic lower bound holds \cite{lattimore2020bandit}:
\begin{equation}\label{eq:borne-classique}
  \liminf_{T\to\infty}
  \inf_{\substack{\tilde\nu_k\in\mathcal{D}\\\mathbb{E}[\tilde\nu_k]>\mu_1}}
  \frac{\mathbb{E}[N_k(T)]\,\KL(\nu_k,\tilde\nu_k)}{\ln T}
  \;\ge\;1,
\end{equation}
results the following lower bound on the regret
\begin{equation*}
    R_T \geq \ln T \left( \sum_{k=2}^K(\mu_1 - \mu_k) \sup_{\substack{\tilde\nu_k\in\mathcal{D}\\\mathbb{E}[\tilde\nu_k]>\mu_1}} \frac{1}{\KL(\nu_k,\tilde\nu_k)} \right) + o(\ln T)
\end{equation*}

\paragraph{Transfer learning setting}  
Before interacting with the target $\nu$, we observe $N'_k$ i.i.d.\ samples from related source distributions $\nu'=(\nu'_1,\dots,\nu'_K)\in\mathcal{D}'^K$ satisfying $d_k(\nu_k,\nu'_k)\le L_k$. The regret on $\nu$ remains $R_T$ as defined above. Our goal is to derive a transfer-aware lower bound analogous to (Eq\eqref{eq:borne-classique}), explicitly involving $(d_k,L_k,N'_k)$, and to design a KL-UCB–type algorithm that matches this bound.

\section{Lower Bounds}\label{sec:LowerBounds}

In this section, we derive fundamental limits on cumulative regret in our offline transfer setting. Specifically, we extend the classical Lai–Robbins bound (Eq\eqref{eq:borne-classique}) to incorporate the transfer parameters $(d_k, L_k, N'_k)$, showing that no consistent algorithm can beat this rate (see Theorem~\ref{thm:lowerb}), and we discuss several important special cases.
\begin{remark}
In the rest of the paper, all $o(\cdot)$ and $O(\cdot)$ notation refer to the limit $T \to \infty$. Since offline sample sizes $N'_k$ may depend arbitrarily on $T$ (or even be $+\infty$), any $o(\cdot)$ or $O(\cdot)$ usage must explicitly account for possible dependence on $N'_k$.
\end{remark}

\subsection{General Lower Bound in the Transfer Setting}

\begin{definition}
An algorithm is said to be \emph{consistent} (w.r.t.\ $(\mathcal{D},\mathcal{D}',d,L,N')$) if for every pair $(\nu,\nu')\in\mathcal{D}^K\times\mathcal{D}'^K$ and every suboptimal arm $k$,
\begin{equation*}
  \forall\,\alpha>0,\quad
  \mathbb{E}[N_k(T)] = o(T^\alpha)
  \quad(T\to\infty).
\end{equation*}
\end{definition}

\begin{theorem}\label{thm:lowerb}
Let a consistent algorithm operate under $(\mathcal{D},\mathcal{D}',d,L,N')$. For any suboptimal arm $k$,
\begin{equation*}
  \liminf_{T\to\infty}
  \inf_{\substack{\tilde\nu_k\in\mathcal{D},\;\tilde\nu_k'\in\mathcal{D}'\\
                  \mathbb{E}[\tilde\nu_k]>\mu_1,\;
                  d_k(\tilde\nu_k,\tilde\nu_k') \le L_k}}
  \frac{\mathbb{E}[N_k(T)]\,\KL(\nu_k,\tilde\nu_k)
        + N_k'\,\KL(\nu_k',\tilde\nu_k')}
       {\ln T}
  \ge 1.
\end{equation*}
\end{theorem}

Equivalently, defining
  $\mathcal{K}_{\inf}(\nu'_k; d_k, \tilde\nu_k, L_k)
  :=
  \inf\left\{
  \KL(\nu'_k,\tilde\nu_k') \mid \ \tilde\nu_k'\in\mathcal{D}',\  d_k(\tilde\nu_k,\tilde\nu_k') \le L_k\right\}$,
\begin{equation*}
  \liminf_{T\to\infty}
  \inf_{\substack{\tilde\nu_k\in\mathcal{D}\\ \mathbb{E}[\tilde\nu_k]>\mu_1}}
  \frac{\mathbb{E}[N_k(T)]\,\KL(\nu_k,\tilde\nu_k)
        + N_k'\,\mathcal{K}_{\inf}(\nu'_k; d_k, \tilde\nu_k, L_k)}
       {\ln T}
  \ge 1.
\end{equation*}

\noindent\textbf{Offline sample size $N'_k$:} We allow $N'_k$ to be any non-negative integer (possibly depending on $T$) or even $+\infty$. If $N'_k=0$ or $N'_k=o(\ln T)$, the offline term is negligible and the equation \eqref{eq:borne-classique} is recovered; if $N'_k=\infty$, then either $\mathcal{K}_{\inf}(\nu'_k; d_k,\tilde\nu_k,L_k)=0$ (reducing to equation \eqref{eq:borne-classique}) or consistency imposes no constraint on $\mathbb{E}[N_k(T)]$, allowing $\mathbb{E}[N_k(T)]=o(\ln T)$.
\\
\\
\textbf{Prior radius $L_k$:} If $L_k\to\infty$, then for any $\,\tilde\nu_k\,$ the constraint $d_k(\tilde\nu_k,\tilde\nu_k')\le L_k$ becomes vacuous, so $\mathcal{K}_{\inf}(\nu'_k; d_k,\tilde\nu_k,L_k)\to0$ and the classical bound is recovered. If $L_k=0$ and $d_k(\nu,\nu')=0\implies\nu=\nu'$, then $\nu_k'=\nu_k$ and
    \begin{equation*}
      \mathbb{E}[N_k(T)]\,\KL(\nu_k,\tilde\nu_k) + N'_k\,\KL(\nu_k',\tilde\nu_k)
      = (\mathbb{E}[N_k(T)] + N'_k)\,\KL(\nu_k,\tilde\nu_k),
    \end{equation*}
    yielding the classical bound \eqref{eq:borne-classique} with effective sample size $\mathbb{E}[N_k(T)] + N'_k$.
\\
\\
\textbf{General implications:} Since $N'_k \,\mathcal{K}_{\inf}(\nu'_k; d_k,\tilde\nu_k,L_k)\ge0$, Theorem~\ref{thm:lowerb} is tighter than the classical bound; prior data on the optimal arm do not affect this lower bound, as the optimal arm is pulled linearly and its mean is estimated arbitrarily well online.

\noindent The proof of Theorem~\ref{thm:lowerb} is deferred to Appendix~\ref{sec:proof_thm1}.

\subsection{Gaussian case with mean‐based distance}\label{sec:partgausslowerbound}

We now instantiate Theorem~\ref{thm:lowerb} when both target and source families are univariate Gaussians with known variances. We use a mean-based distance:
\begin{equation*}
  \mathcal{D} = \{\mathcal{N}(\mu,\sigma^2)\},\quad
  \mathcal{D}' = \{\mathcal{N}(\mu,\sigma'^2)\},\quad
  d_k(\mathcal{N}(\mu,\sigma^2),\mathcal{N}(\tilde\mu,\sigma'^2))
    = |\mu - \tilde\mu|.
\end{equation*}
Hence the constraint $|\tilde\mu_k - \tilde\mu_k'|\le L_k$. We will also note $(x)_+ := \max\{x,0\}$.

\begin{corollary}[Gaussian lower bound]
Under the Gaussian setting with known variances and mean‐based distance defined above, any consistent algorithm satisfies for each suboptimal arm $k$:
\begin{equation}\label{cor:optimalitegauss}
  \liminf_{T\to\infty}
  \frac{\mathbb{E}[N_k(T)]}{\ln T}\,\frac{(\mu_1-\mu_k)^2}{2\,\sigma^2}
  + \frac{N_k'}{\ln T}\,\frac{(\mu_1-\mu_k'-L_k)_+^2}{2\,\sigma'^2}
  \ge 1.
\end{equation}
\end{corollary}

\begin{proof}[Sketch of proof]
Apply Theorem~\ref{thm:lowerb} with
\begin{equation*}
  \KL(\mathcal{N}(\mu,\sigma^2), \mathcal{N}(\tilde\mu,\sigma^2))
  = \frac{(\mu-\tilde\mu)^2}{2\,\sigma^2},
\end{equation*}
and similarly for source-to-source KL. Constraining $\mathbb{E}[\tilde\nu_k]>\mu_1$ yields the first term in the bound, while enforcing $|\tilde\mu_k - {\tilde \mu}_k'|\le L_k$ produces the $(\cdot)_+$ in the second term.
\end{proof}
Equivalently, solving for the expected number of pulls yields:
\begin{equation}\label{eq:optimalitegauss2}
  \mathbb{E}[N_k(T)]
  \ge
  \frac{2\,\sigma^2}{(\mu_1-\mu_k)^2}
  \Bigl(\ln T
    - N_k'\,\tfrac{(\mu_1-\mu_k'-L_k)_+^2}{2\,\sigma'^2}\Bigr)_+
  + o(\ln T).
\end{equation}

\section{KL-UCB Algorithm}

In this section we first recall the classical \textsc{KL-UCB} algorithm, we then introduce our \textsc{KL-UCB-Transfer} variant, which augments the usual index with an extra penalty term to account for prior offline samples and will match Equation \ref{eq:optimalitegauss2}.

\subsection{Classical \textsc{KL-UCB}}

For the classical bandit problem, the \textsc{KL-UCB} algorithm introduced in \cite{pmlr-v19-maillard11a} matches the asymptotic lower bound (Eq \ref{eq:borne-classique}). The core idea of this algorithm is to select the arm with the highest potential mean. To do so, we define an index $U_a(t)$ for each arm $a$, and pull the arm with the highest index:
\[
U_a(t) := \max \left\{ \mathbb{E}[\nu] \ \middle| \ N_a(t) \, \KL(\pi_{\mathcal{D}}(\hat\nu_a(t)), \nu) \leq \delta_t,\ \nu \in \mathcal{D} \right\}
\]
where $\pi_{\mathcal{D}}(\hat\nu_k(t))$ is the projection of the empirical distribution $\hat\nu_k(t)$ onto the space of admissible distributions $\mathcal{D}$. This algorithm is optimal for bounded reward distributions (see \cite{pmlr-v19-garivier11a, Cappe2013KL}).
We typically choose $\delta_t$ to be of the order of $\ln t$, possibly with an additional lower-order term depending on $\mathcal{D}$.

\subsection{\textsc{KL-UCB-Transfer}}\label{sec:KLUCB}

In our transfer learning setting, assuming Gaussian distributions with known variance as in Section \ref{sec:partgausslowerbound}, we define the cost function with respect to the prior data as follows:
\begin{equation*}
    f_a^+(q) := N'_a \frac{\left( q - (\hat \mu'_a(t) + L_a) \right)_+^2}{2 \sigma'^2}
\end{equation*}

In this setting, the index becomes:
\begin{equation}\label{eq:indiceKL}
    U_a(t) := \max \left\{ q \ \middle| \ N_a(t)\frac{( q - \hat \mu_a(t))_+^2}{2 \sigma^2} + f_a^+(q) \leq \delta_t \right\},
\end{equation}

with $\delta_t = \ln t + 3\ln(\max(1,\ln t))$, we compute $U_a(t)$ each turn and pull the arm with highest index; the positive part ensures a non-empty set. An explicit form of $U_a(t)$ is given in the supplement, enabling $O(1)$ computation. The idea behind the algorithm is to incorporate the cost of deviating from the prior information into the estimation of the highest potential mean. Thus, we recover the lower bound given in section \ref{sec:partgausslowerbound}.

\begin{remark}
    If there is no prior data, \textsc{KL-UCB-Transfer} reverts to classical \textsc{KL-UCB}.
\end{remark}

\paragraph{Design challenges and points of attention}
The key challenge in \textsc{KL-UCB-Transfer} is to balance prior and online data without merging them: unlike \cite{rahul2024transfer}, which fuses means from different but similar distributions, we keep empirical and prior means separate, using the positive part in $f_a^+(q)$ to ensure a non-empty confidence set while penalizing exceedances above the prior.

\section{Regret bound}\label{sec:RegretKL}

Working in the setting of section \ref{sec:partgausslowerbound} with our \textsc{KL-UCB-Transfer} algorithm, we have control over the number of pulls for each suboptimal arm through the following theorem.
\begin{theorem}[\textsc{KL-UCB-Transfer} Regret Bound]\label{thm:KLREGRET}
Suppose for each arm \(k\) we have \(N_k'\) i.i.d.\ samples from 
\(\nu_k'=\mathcal{N}(\mu_k',\sigma'^2)\) and that the true target law is 
\(\nu_k=\mathcal{N}(\mu_k,\sigma^2)\) with \(\lvert\mu_k-\mu_k'\rvert\le L_k\).  
Then under \textsc{KL-UCB-Transfer}, every suboptimal arm \(a\) satisfies
\[
\mathbb{E}[N_a(T)]
\;\le\;
\frac{2\,\sigma^2}{(\mu_1-\mu_a)^2}
\left(\ln T 
- N_a'\,\frac{\bigl(\mu_1-(\mu_a'+L_a)\bigr)_+^2}{2\,\sigma'^2}\right)_+
+O\bigl((\ln T)^{2/3}\bigr).
\]
\end{theorem}

Hence, we recover optimality according to Eq \eqref{eq:optimalitegauss2}. The exact expression of the \(O\bigl((\ln T)^{2/3}\bigr)\) term can be derived in the proof of Theorem~\ref{thm:KLREGRET} in the Appendix \ref{sec:proof_thm2}. The main challenge is handling $N'_a$, which may be arbitrarily large and depend on $T$.

\section{Experiments}\label{sec:sim}

We evaluate \textsc{KL-UCB-Transfer} in three simulations on a 6-armed bandit with unit-variance Gaussian rewards and true means
\[
\mu = (1.0,\;0.9,\;0.8,\;0.7,\;0.6,\;0.5),
\]
so arm~1 is optimal. When injecting prior data on arm~$k$, we use $N'_k=1000$ synthetic samples from \(\mathcal{N}(\mu'_k,1)\) to observe a noticeable improvement in the lower bound (Eq.~\eqref{eq:optimalitegauss2}).

\paragraph{Exploration bonus in simulations}
Instead of \(\delta_t=\log t+3\log\log t\), we set \(\delta_t=(1+\varepsilon)\log t\) with \(\varepsilon=\tfrac1{20}\), as in \cite{Cappe2013KL}. This reduces over-exploration (\(\varepsilon\log T \le 3\log\log T\) for \(T<10^{152}\)) and inflates the \(\log T\) regret term by only a factor \(1+\varepsilon\). Adapting the theory to this choice of \(\delta_t\) is straightforward, and simulations confirm its improved performance.

\subsection*{Simulation 1: Priors on all arms}

In this study we apply the \emph{same} prior shift $(\delta,L)$ to every arm. A prior on arm $k$ yields asymptotic improvement whenever $\mu_1 > \mu'_k + L_k.$ Concretely, we set
\[
\mu'_k = \mu_k + \delta, \quad L_k = L \quad\text{for }k=1,\dots,6,
\]
and sweep through four increasingly accurate settings:
\[
(\delta,L)\in\{(0.20,0.40),\;(0.11,0.20),\;(0.05,0.10),\;(0.00,0.05)\}.
\]
We plot five curves: the no-prior baseline ($N'_k = 0$ for all arms) plus one for each $(\delta,L)$. In the first setting $(0.20,0.40)$, for every suboptimal arm $k$ we have $\mu_1 < \mu'_k + L_k$, so there is no long-term regret benefit. In the last three (tighter) settings, some arms satisfy $\mu_1 > \mu'_k + L_k$, hence those priors do improve long-term regret. Figure~\hyperref[fig:pic1]{1(a)}  shows that only when $(\delta,L)$ is small enough does \textsc{KL-UCB-Transfer} outperform the no-prior baseline.

\subsection*{Simulation 2: Priors only on the optimal arm}

As shown in (Eq \eqref{eq:indiceKL}), adding prior data for arm~$a$ reduces its index $U_a(t)$; hence a prior on a suboptimal arm is beneficial, while a prior on the optimal arm may not be. In our simulations, we focus on the effect of a prior placed only on the optimal arm. We restrict priors to the optimal arm ($k=1$) and compare three configurations:

\noindent$\bullet$ No prior (baseline).\\
$\bullet$  Mildly optimistic (Prior 1): $\mu'_1 = \mu_1 + 0.001,\;L_1 = 0.004$.\\
$\bullet$  Pessimistic (Prior 2):       $\mu'_1 = \mu_1 - 0.010,\;L_1 = 0.210$.

\noindent Define the hardship parameter $\eta := L_1 - (\mu'_1 - \mu_1)$ (see proof of Theorem~\ref{thm:KLREGRET}). Although both priors are asymptotically harmless, a small $\eta$ significantly slows early learning. Figure~\hyperref[fig:pic2]{1(b)} illustrates that a very small $\eta$ (mildly optimistic prior) leads to higher initial regret, whereas a larger $\eta$ (pessimistic prior) grants an early advantage, even though both settings converge to the same asymptotic performance.

\vspace{-2pt}
\begin{figure}[htbp]
\centering
{\subfigure[Simulation 1 with Regret up to $T=10^6$. Same priors on all arms.]{%
\label{fig:pic1}
\includegraphics[width=0.32\textwidth]{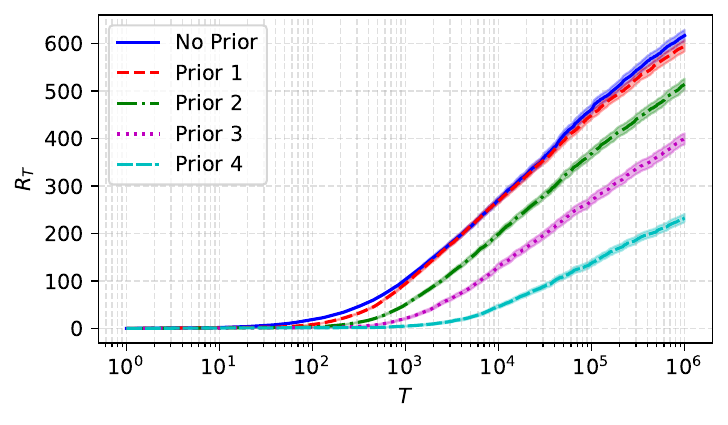}}}
{
\subfigure[Simulation 2 with Regret up to $T=10^4$. Priors on the optimal arm only.\label{fig:pic2}]{%
\includegraphics[width=0.31\textwidth]{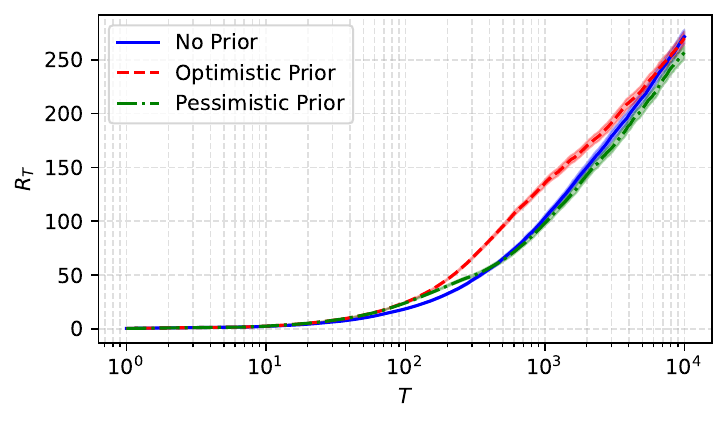}
}}{
\subfigure[Simulation 3, with Regret up to $T=10^6$ with prior $(\delta,L)=(0.05,0.10)$.  \label{fig:sim3}
]{
  \includegraphics[width=0.32\textwidth]{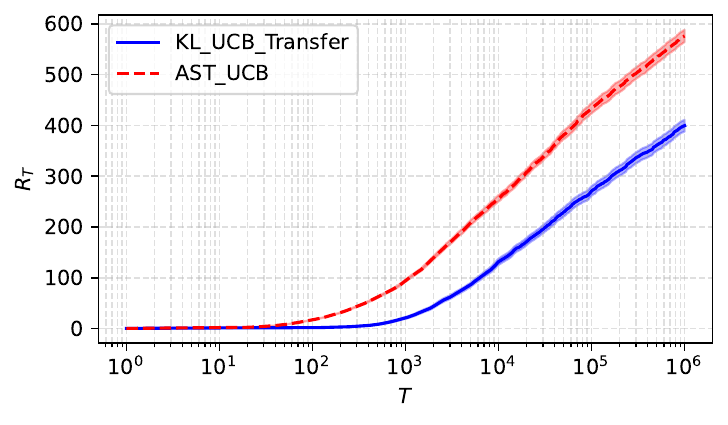}
  }}
  \caption{100 runs for each plot. Error bars show standard error of the mean (i.e.\ sample standard deviation divided by $\sqrt{100}$) \label{fig:plot1}}
\end{figure}
\vspace{-4pt}

\subsection*{Simulation 3: Comparison with AST-UCB}

We compare our algorithm \textsc{KL-UCB-Transfer} to \textsc{AST-UCB} from \cite{rahul2024transfer}.  \textsc{AST-UCB} is designed for an \emph{episodic} non-contextual bandit: the learner plays a sequence of episodes on the same arms, where at the start of each episode the arm means may shift by at most $L$ from their values in the previous episode; \textsc{AST-UCB} uses past-episode data to compute the index in the new episode. In this simulation we use the third prior setting from Simulation 1, namely $(\delta,L)=(0.05,0.10)$ on all arms, since it tightens the lower bound (Eq.~\eqref{eq:optimalitegauss2}) and allows us to compare how efficiently our algorithm and \textsc{AST-UCB} leverage the prior data. We run \textsc{AST-UCB} with its matching episodic-shift parameter $L=0.10$. Figure~\hyperref[fig:sim3]{1(c)} demonstrates that \textsc{KL-UCB-Transfer} yields substantially lower regret than \textsc{AST-UCB} under this prior.

\paragraph{Takeaway message}
Overall, these simulations confirm that \textsc{KL-UCB-Transfer} can harness identical priors on every arm to reduce regret when the priors are accurate, reveal the nuanced short‐term trade‐offs when the prior is applied solely to the optimal arm, and consistently outperforms existing baselines in this transfer setting.

\section{Conclusion}

\noindent\textbf{Asymptotic Lower Bound and KL-UCB-Transfer.}
In this work we have derived a sharp, problem-dependent asymptotic lower bound on cumulative regret that holds for any algorithm operating under arbitrary offline prior data.  Focusing on the Gaussian bandit setting, we then introduced \textsc{KL-UCB-Transfer}, a simple modification of KL-UCB that incorporates priors via an extra KL penalty.  We proved that \textsc{KL-UCB-Transfer} \emph{matches} the general lower bound exactly when all arms are Gaussian with known variance and the prior samples $N'_k$ are fixed, i.i.d.\ observations.

\medskip
\noindent\textbf{Simulation Results and Trade-offs.}
Our simulations confirm that, in regimes where the priors on suboptimal arms are sufficiently accurate, \textsc{KL-UCB-Transfer} enjoys strictly lower regret than the classical, no-prior KL-UCB.  We also highlighted the short-term trade-offs that arise when one supplies priors on the optimal arm.

\medskip
\noindent\textbf{Future Directions.}
Looking forward, it would be interesting to extend our analysis beyond Gaussian rewards to exponential or bounded families, and to design transfer-aware algorithms that optimally balance prior bias against data-driven exploration in non-stationary or structured bandit environments.  Moreover, in our current setting the prior sample sizes $N'_k$ are non-random and the prior data are assumed i.i.d.  A promising direction for future work is to study the case where prior data come from another online bandit process.

More broadly, our results offer a clean and tractable foundation to better understand transfer in simple settings, which can serve as a stepping stone toward principled transfer algorithms in richer and more realistic environments.


\bibliography{acml25}

\appendix
\vspace{-0.3cm}
\section{Proof of Theorem \ref{thm:lowerb}}\label{sec:proof_thm1}

We now prove Theorem \ref{thm:lowerb}. For this, we begin with the following lemma.

\begin{lemma}\label{lem:lemlower}
Let $(\nu,\nu')$ and $(\tilde{\nu},\tilde{\nu}')$ be two sets of arm distributions in our setting. Denote by $\mathbb{P}_{(\nu,\nu')\pi,T}$ the canonical probability measure of the bandit model using strategy $\pi$ and arm distributions $(\nu,\nu')$ up to time $T$. Then,
\begin{equation*}
    \KL(\mathbb{P}_{(\nu,\nu')\pi,T},\mathbb{P}_{(\tilde{\nu},\tilde{\nu}')\pi,T}) = \sum_{k = 1}^K \mathbb{E}_{(\nu,\nu')}\left[N_k(T) \right] \KL(\nu_k,\tilde{\nu}_k) + \sum_{k = 1}^K N'_k \KL(\nu_k',\tilde{\nu}_k').
\end{equation*}
\end{lemma}

\noindent This lemma parallels Lemma 15.1 in \cite{lattimore2020bandit}, except for the term $\sum_{k=1}^K N'_k \KL(\nu'_k,\tilde{\nu}'_k)$ accounting for initial offline data.

\begin{proof}[of Theorem~\ref{thm:lowerb}]

Let $k$ be a suboptimal arm for $(\nu,\nu')$. Consider a modification $(\tilde{\nu},\tilde{\nu}') \in \mathcal{D} \times \mathcal{D}'$ such that arm $k$ becomes optimal, while keeping the distributions $(\nu_i,\nu_i')_{i \neq k}$ unchanged. Then, using the data-processing inequality (e.g., \cite[Lemma 3.3]{maillard:tel-02162189}) combined with the KL decomposition of Lemma \ref{lem:lemlower}, we get:
\begin{align*}
    \mathbb{E}_{(\nu,\nu')}\left[N_k(T)\right] \KL(\nu_k,\tilde{\nu}_k) + N'_k \KL(\nu_k',\tilde{\nu}_k') &\geq kl\left( \frac{\mathbb{E}_{(\nu,\nu')}[N_k(T)]}{T} , \frac{\mathbb{E}_{(\tilde{\nu},\tilde{\nu}')}[N_k(T)]}{T}\right).
\end{align*}
Now using that for all $p,q \in (0,1)$ we have
\begin{equation*}
    kl(p,q)  = (1-p)\ln\left( \frac{1}{1-q} \right) + \underbrace{p \ln (1/q)}_{\geq 0} +  \underbrace{p \ln (p) + (1-p)\ln(1-p)}_{\geq - \ln 2},
\end{equation*}
we have that $\mathbb{E}_{(\nu,\nu')}\left[N_k(T)\right] \KL(\nu_k,\tilde{\nu}_k) + N'_k \KL(\nu_k',\tilde{\nu}_k')$ is lower bounded by
\begin{equation*}
    \left(1 - \frac{\mathbb{E}_{(\nu,\nu')}[N_k(T)]}{T}\right) \ln \left( \frac{T}{T - \mathbb{E}_{(\tilde{\nu},\tilde{\nu}')}[N_k(T)]} \right) - \ln(2)
\end{equation*}
Using the consistency of the strategy $\pi$, we have for any $\alpha > 0$:
\begin{itemize}
    \item $\mathbb{E}_{(\nu,\nu')}[N_k(T)] = o(T^\alpha)$,
    \item $T - \mathbb{E}_{(\tilde{\nu},\tilde{\nu}')}[N_k(T)] = o(T^\alpha)$.
\end{itemize}

Therefore, for any $\alpha > 0$,
\begin{equation*}
    \frac{\mathbb{E}_{(\nu,\nu')}[N_k(T)] \KL(\nu_k,\tilde{\nu}_k) + N'_k \KL(\nu_k',\tilde{\nu}_k')}{\ln T} \geq (1 - o(T^{-(1-\alpha)}))(1 - \alpha) + o_{T \to \infty}(1).
\end{equation*}

Since this is true for all $\alpha > 0$, taking the infimum over all $(\tilde{\nu},\tilde{\nu}') \in \mathcal{D} \times \mathcal{D}'$ such that arm $k$ becomes optimal (with all other arms unchanged) yields the result in Theorem \ref{thm:lowerb}.
\end{proof}
\vspace{-0.6cm}
\subsection*{Proof of Lemma \ref{lem:lemlower}}

\begin{proof}

We follow the same idea as in the proof of Lemma 15.1 from \cite{lattimore2020bandit}. Let $p_{(\nu,\nu')\pi,T}$ denote the density of $\mathbb{P}_{(\nu,\nu')\pi,T}$, the probability of the space of $T$ turn under the policy $\pi$ with the arm distribution $(\nu,\nu')$. Let $x'$ be the offline data. Then,
\begin{align*}
    p_{(\nu,\nu')\pi,T}(x',a_1,x_1,...,a_T,x_T) &= p_{\nu'}(x') \prod_{t = 1}^{T}\pi_t(a_t \mid x',a_1,x_1,...,a_{t-1},x_{t-1})p_{\nu_{a_t}}(x_t) \\
    &= \prod_{k = 1}^{K}\prod_{i = 1}^{N'_k} p_{\nu'_k}(x'_{i,k}) \prod_{t = 1}^{T}\pi_t(a_t \mid x',a_1,x_1,...,a_{t-1},x_{t-1})p_{\nu_{a_t}}(x_t).
\end{align*}

Taking the likelihood ratio with $p_{(\tilde{\nu},\tilde{\nu}')\pi,T}$ gives:
\begin{align*}
    \frac{p_{(\nu,\nu')\pi,T}(x',a_1,x_1,...,a_T,x_T)}{p_{(\tilde{\nu},\tilde{\nu}')\pi,T}(x',a_1,x_1,...,a_T,x_T)} &= \prod_{k = 1}^{K}\prod_{i = 1}^{N'_k} \frac{p_{\nu'_k}(x'_{i,k})}{p_{\tilde{\nu}'_k}(x'_{i,k})} \prod_{t = 1}^{T}\frac{p_{\nu_{a_t}}(x_t)}{p_{\tilde{\nu}_{a_t}}(x_t)}.
\end{align*}

Taking the log of both sides:
\begin{align*}
    \ln \frac{\mathbb{P}_{(\nu,\nu')\pi,T}}{\mathbb{P}_{(\tilde{\nu},\tilde{\nu}')\pi,T}}(x',a_1,x_1,...,a_T,x_T) &= \sum_{k = 1}^{K}\sum_{i = 1}^{N'_k} \ln \frac{p_{\nu'_k}(x'_{i,k})}{p_{\tilde{\nu}'_k}(x'_{i,k})} + \sum_{t = 1}^{T} \ln \frac{p_{\nu_{a_t}}(x_t)}{p_{\tilde{\nu}_{a_t}}(x_t)}.
\end{align*}

Taking expectation under $(\nu,\nu')$ yields:
\begin{align*}
    \mathbb{E}_{(\nu,\nu')} \left[\ln \frac{\mathbb{P}_{(\nu,\nu')\pi,T}}{\mathbb{P}_{(\tilde{\nu},\tilde{\nu}')\pi,T}} \right] &= \sum_{k = 1}^{K} N'_k \KL(\nu'_k,\tilde{\nu}'_k) + \sum_{t = 1}^{T} \mathbb{E}_{(\nu,\nu')} \left[ \KL(\nu_{a_t}, \tilde{\nu}_{a_t}) \right] \\
    &= \sum_{k = 1}^{K} N'_k \KL(\nu'_k,\tilde{\nu}'_k) + \sum_{t = 1}^{T} \mathbb{E}_{(\nu,\nu')} \left[ \sum_{k = 1}^K \mathbf{1}(A_t = k) \KL(\nu_k,\tilde{\nu}_k) \right] \\
    &= \sum_{k = 1}^{K} N'_k \KL(\nu'_k,\tilde{\nu}'_k) + \sum_{k = 1}^K \mathbb{E}_{(\nu,\nu')} \left[ N_k(T) \right] \KL(\nu_k,\tilde{\nu}_k).
\end{align*}
\end{proof}
\vspace{-0.7cm}
\section{Proof of Theorem \ref{thm:KLREGRET}}\label{sec:proof_thm2}

We will use the shorthand
\[
  f^+_a(q)=N'_a\frac{\bigl(q-(\hat\mu'_a+L_a)\bigr)_+^2}{2\,\sigma'^2},\quad
  d(a,b)=\frac{(b-a)^2}{2\,\sigma^2},\quad
  d^+(a,b)=\frac{(b-a)_+^2}{2\,\sigma^2}.
\]
We use the following lemmas. Lemma~\ref{lem:lemma1} is proved in Section~\ref{sec:proof_lemma1}; proof sketches for Lemmas~\ref{lem:lemma2}, \ref{lem:lemma3}, and \ref{lem:lemma4} appear in Sections~\ref{sec:proof_lemma2}, \ref{sec:proof_lemma3}, and \ref{sec:proof_lemma4}, with full proofs in the Supplement.

\begin{lemma} \label{lem:lemma1}For all $T > 1$, for all arm $a \neq 1$, we have the inequality
\[
\sum_{t = 1}^{T}\mathbf{1}\left( A_t = a,\ \mu_1 \leq U_1(t) \right) \leq \sum_{s = 1}^{T} \mathbf{1}\left( s\, d^+(\hat \mu_{a,s},\mu_1) + f^+_a(\mu_1) \leq \delta_T \right).
\]
\end{lemma}

\begin{lemma}\label{lem:lemma2}
For every $\delta_t \underset{T \to +\infty}{\longrightarrow} + \infty$ and for all arm $a \neq 1$, we have the inequality
\[
\mathbb{E}\left[(\delta_T - f^+_a(\mu_1))_+\right] \leq \left(\delta_T - N_a'\frac{\left( \mu_1 - (\mu_a' + L_a) \right)_+^2}{2 \sigma'^2}\right)_+ + O(\sqrt{\delta_T}).
\]
\end{lemma}

\begin{lemma}\label{lem:lemma3}
There exists $ C' > 0$ that depends only on $\sigma,\sigma',L_1,\mu_1,\mu'_1$, such that:
\[
\mathbb{E}\left[ \exp \left( f^+_1(\mu_1) \right) \right] \leq C'.
\]
\end{lemma}

\begin{lemma}\label{lem:lemma4}
    For all $t \geq 1$, using the choice $\delta_t := \ln t + 3 \ln(\max(1,\ln t))$,
    \begin{equation*}
        \mathbb{P}(\mu_1>U_1(t)) = O\left(\frac{\mathbb{E}[e^{f^+_1(\mu_1)}]}{t\ln t}\right).
    \end{equation*}
\end{lemma}

We bound the expected pulls of a suboptimal arm $a$ by splitting into two cases:

1. Times when the upper confidence bound of the optimal arm underestimates its true mean ($\mu_1 > U_1(t)$). By choosing $\delta_t$ appropriately and using Lemma \ref{lem:lemma3} and \ref{lem:lemma4}, this case contributes only $o(\ln T)$ to the regret.

2. Times when $U_1(t)$ is above $\mu_1$ but arm $a$ is still selected. Here one applies Lemma \ref{lem:lemma1} to show that after a random threshold $M$, further pulls of $a$ become unlikely, yielding the main $\ln T$ order term (up to lower-order corrections from Lemma \ref{lem:lemma2}).

Combining these two bounds gives the desired logarithmic bound on $\mathbb{E}[N_a(T)]$.
\\
\noindent The changes in the proof, compared to that in \cite{pmlr-v19-garivier11a,Cappe2013KL}, include additional conditioning steps, Lemma~\ref{lem:lemma2}, and Lemma~\ref{lem:lemma3}, which arise from the inclusion of prior data.

\begin{proof}[of Theorem \ref{thm:KLREGRET}]

We begin by bounding the expected number of pulls of a suboptimal arm $ a $:
\begin{equation}\label{eq:eq1} 
    \mathbb{E}\left[ N_a(T) \right] = \mathbb{E}\left[\sum_{t = 1}^{T} \mathbf{1}\left( A_t = a \right) \right] \leq \sum_{t = 1}^{T} \mathbb{P}\left( \mu_1 > U_1(t) \right)  + \mathbb{E}\left[\sum_{t = 1}^{T} \mathbf{1}\left( A_t = a,\ \mu_1 \leq U_1(t) \right) \right] 
\end{equation}
Now using Lemmas \ref{lem:lemma3} and \ref{lem:lemma4} we get:
\begin{equation}\label{eq:eq2}
    \sum_{t = 1}^{T} \mathbb{P}\left( \mu_1 > U_1(t) \right) = \sum_{t = 1}^{T}O\left( \frac{1}{t \ln t} \right) = O(\ln \ln T)
\end{equation}

For the second term in \eqref{eq:eq1}, using Lemma \ref{lem:lemma1}, define the random quantity:
\[
M := \frac{(\delta_T - f^+_a(\mu_1))_+}{(1 - \varepsilon)^2 d(\mu_a, \mu_1)}
\]
Then:
\begin{align}
\mathbb{E}&\left[ \sum_{t = 1}^{T} \mathbf{1}\left( A_t = a,\ \mu_1 \leq U_1(t) \right) \right] \\
&\stackrel{(a)}{=} \mathbb{E}\left[ \sum_{s = 1}^{T} \mathbb{P}\left( s\, d^+(\hat \mu_{a,s}, \mu_1) + f^+_a(\mu_1) \leq \delta_T \Big| f^+_a \right) \right] \nonumber \\
&\leq \mathbb{E}\left[ M - 1 + \sum_{s = M}^{T} \mathbb{P}\left( s\, d^+(\hat \mu_{a,s}, \mu_1) + f^+_a(\mu_1) \leq \delta_T \Big| f^+_a \right) \right] \nonumber \\
&\leq \mathbb{E}\left[ M - 1 + \sum_{s = M}^{T} \mathbb{P}\left( d^+(\hat \mu_{a,s}, \mu_1) \leq \frac{\delta_T - f^+_a(\mu_1)}{M} \Big| f^+_a \right) \right] \nonumber \\
&\stackrel{(b)}{\leq} \mathbb{E}\left[ M - 1 + \sum_{s = M}^{T} \mathbb{P}\left( d^+(\hat \mu_{a,s}, \mu_1) \leq (1 - \varepsilon)^2 d(\mu_a, \mu_1) \right) \right] \nonumber \\
&\leq \frac{\mathbb{E}\left[ (\delta_T - f^+_a(\mu_1))_+ \right]}{(1 - \varepsilon)^2 d(\mu_a, \mu_1)} - 1 + \sum_{s = 1}^{T} \mathbb{P}\left( d^+(\hat \mu_{a,s}, \mu_1) \leq (1 - \varepsilon)^2 d(\mu_a, \mu_1) \right), \label{eq:KLregreteq1}
\end{align}
where $(a)$ follows by Lemma \ref{lem:lemma1} and by conditioning and $(b)$ by the definition of $M$.\\
We now apply a standard Chernoff bound for Gaussian tails for the first inequality:
\begin{align}
\sum_{s = 1}^{T} \mathbb{P}\left( d^+(\hat \mu_{a,s}, \mu_1) \leq (1 - \varepsilon)^2 d(\mu_a, \mu_1) \right)
&= \sum_{s = 1}^{T} \mathbb{P}\left( \hat \mu_{a,s} \geq \mu_a + \varepsilon (\mu_1 - \mu_a) \right)\nonumber \\
&\leq \sum_{s = 1}^{T} \exp\left( - \frac{\varepsilon^2 (\mu_1 - \mu_a)^2}{2 \sigma^2} s \right)\nonumber \\
& \leq \frac{1}{1 - \exp\left( - \frac{\varepsilon^2 (\mu_1 - \mu_a)^2}{2 \sigma^2} \right)}\nonumber \\
& = O(\varepsilon^{-2}).\label{eq:KLregreteq2}
\end{align}

Using Equations \eqref{eq:KLregreteq1}, \eqref{eq:KLregreteq2} and Lemma \ref{lem:lemma2}, we finally obtain
\begin{multline*} 
\mathbb{E}\left[ \sum_{t = 1}^{T} \mathbf{1}\left( A_t = a,\ \mu_1 \leq U_1(t) \right) \right]\le \frac{\left( \delta_T - N_a'\frac{\left( \mu_1 - (\mu_a' + L_a) \right)_+^2}{2 \sigma'^2} \right)_+}{(1 - \varepsilon)^2 d(\mu_a, \mu_1)}+ O(\varepsilon^{-2}) + O(\sqrt{\ln T})  \\
 = \frac{2 \sigma^2}{(\mu_1 - \mu_k)^2}\left(\ln T - N_a'\frac{\left( \mu_1 - (\mu_a' + L_a) \right)_+^2}{2 \sigma'^2}\right)_+ + O(\varepsilon\ln T)+ O(\varepsilon^{-2}) + O(\sqrt{\ln T}).
\end{multline*}
By taking $\varepsilon = (\ln T)^{-\frac{1}{3}}$ we get
\begin{equation}\label{eq:eq3}
    \mathbb{E}\left[ \sum_{t = 1}^{T} \mathbf{1}\left( A_t = a,\ \mu_1 \leq U_1(t) \leq \right) \right] \leq \frac{2 \sigma^2}{(\mu_1 - \mu_k)^2}\left(\ln T - N_a'\frac{\left( \mu_1 - (\mu_a' + L_a) \right)_+^2}{2 \sigma'^2}\right)_+ + O\left((\ln T)^\frac{2}{3}\right).
\end{equation}
Combining Equations \eqref{eq:eq1}, \eqref{eq:eq2}, and \eqref{eq:eq3}, we conclude that,
\[
\mathbb{E}\left[ N_a(T) \right] \leq \frac{2 \sigma^2}{(\mu_1 - \mu_k)^2}\left(\ln T - N_a'\frac{\left( \mu_1 - (\mu_a' + L_a) \right)_+^2}{2 \sigma'^2}\right)_+ + O\left((\ln T)^\frac{2}{3}\right).
\]
\end{proof}

\vspace{-1cm}
\subsection{Proof of Lemma \ref{lem:lemma1}}\label{sec:proof_lemma1}

The proof mimes the one of Lemma 7 in \cite{pmlr-v19-garivier11a}

\begin{proof}
    If $A_t = a$ and $\mu_1 \leq U_1(t)$, then $U_a(t) \geq U_1(t) \geq \mu_1$, and therefore:
\begin{equation*}
    N_a(t) d^+(\hat \mu_a(t),\mu_1) + f^+_a(\mu_1) \leq N_a(t) d^+(\hat \mu_a(t),U_a(t)) + f^+_a(U_a(t)) = \delta_t \leq \delta_T
\end{equation*}
Hence,
\begin{align*}
    \sum_{t = 1}^{T}\mathbf{1}\left( A_t = a, \mu_1 \leq U_1(s)  \right) & \leq \sum_{t = 1}^{T}\mathbf{1}\left( A_t = a, N_a(t) d^+(\hat \mu_a(t),\mu_1) + f^+_a(\mu_1) \leq \delta_T  \right) \\
    & = \sum_{t = 1}^{T}\sum_{s = 1}^{t}\mathbf{1}\left(N_a(t) = s, A_t = a, s d^+(\hat \mu_{a,s},\mu_1) + f^+_a(\mu_1) \leq \delta_T  \right) \\
    & = \sum_{s = 1}^{T}\mathbf{1} \left( s d^+(\hat \mu_{a,s},\mu_1) + f^+_a(\mu_1) \leq \delta_T\right) \sum_{t = s}^{T}\mathbf{1} \left( N_a(t) = s, A_t = a\right) \\
    & \leq \sum_{s = 1}^{T}\mathbf{1} \left( s d^+(\hat \mu_{a,s},\mu_1) + f^+_a(\mu_1) \leq \delta_T\right)
\end{align*}
\end{proof}

\vspace{-1cm}
\subsection{Proof Sketch of Lemma \ref{lem:lemma2}}\label{sec:proof_lemma2}
\begin{proof}[Proof Sketch]
Let 
\begin{equation*}
    X := \delta_t - \frac{N'_a}{2\sigma'^2}\bigl(\mu_1 - (\hat\mu'_a + L_a)\bigr)_+^2,
\quad
\xi := \mu_1 - L_a - \mu'_a,\;
\beta := \xi\sqrt{N'_a/\sigma'^2}.
\end{equation*}
Using that $\hat\mu'_a$ is gaussian we got 
$\mathbb{E}[X_+] = \displaystyle\int_{0}^{\delta_t} \Phi\bigl(\sqrt{2(\delta_t - t)} - \beta\bigr)\,dt =: I(\beta,\delta_t).$
Writing $\delta=\delta_t$, after some calculations one obtains the closed form
\begin{equation*}
    I(\beta,\delta)
=
\frac{2\delta - \beta^2 - 1}{2}\,\Phi(\sqrt{2\delta}-\beta)
+\frac{\beta^2 + 1}{2}\,\Phi(-\beta)
+\frac{\sqrt{2\delta} + \beta}{2}\,\phi(\sqrt{2\delta}-\beta)
-\frac{\beta}{2}\,\phi(\beta).
\end{equation*}
Analyzing the three regimes $\beta \le 0$, $0 \le \beta \le \sqrt{2\delta}$, and $\beta \ge \sqrt{2\delta}$ gives
\begin{equation*}
    I(\beta,\delta)\;\le\;\left(\delta - \tfrac{\beta_+^2}{2}\right)_+ \;+\; O(\sqrt{\delta}).
\end{equation*}

\end{proof}
\vspace{-1cm}
\subsection{Proof Sketch of Lemma \ref{lem:lemma3}}\label{sec:proof_lemma3}


\begin{proof}[Proof Sketch]
Let \(Z=\mu'_1-\hat\mu'_1\sim\mathcal{N}(0,\sigma'^2/N'_1)\) and set 
\[
\eta :=L_1 - (\mu'_1 - \mu_1) ,\quad a \;:=\;\frac{\eta}{\sigma'/\sqrt{N'_1}}>0,\quad 
W:=\frac{Z\sqrt{N'_1}}{\sigma'}\sim\mathcal{N}(0,1).
\]
Then 
\[
f^+_1(\mu_1)=\tfrac{N'_1}{2\sigma'^2}\bigl(\mu_1-(\hat\mu'_1+L_1)\bigr)_+^2
=\tfrac12\,(W-a)_+^2,
\]
so
\[
\mathbb{E}\bigl[e^{f^+_1(\mu_1)}\bigr]
=\int_{-\infty}^a\varphi(w)\,dw
+\int_a^\infty e^{\tfrac12(w-a)^2}\varphi(w)\,dw
=\Phi(a)+\int_a^\infty\frac{1}{\sqrt{2\pi}}e^{-aw+\tfrac{a^2}{2}}dw
=\Phi(a)+\frac{\varphi(a)}{a}.
\]
Since \(a\ge\frac{\eta}{\sigma'}> 0\), \(\Phi(a)+\varphi(a)/a\le1+\frac{\sigma'}{\eta\sqrt{2 \pi} }=:C'<\infty\). This completes the proof.
\end{proof}

\vspace{-0.5cm}
\subsection{Proof Sketch of Lemma \ref{lem:lemma4}}\label{sec:proof_lemma4}


\begin{proof}[Proof Sketch]
Denote for simplicity $f:=f^+_1(\mu_1)$. Then
\[
\mathbb{P}(\mu_1>U_1(t))
= \mathbb{E}\bigl[\mathbb{P}(\mu_1>U_1(t)\mid f)\bigr]
\le \mathbb{E}\Bigl[1_{f\ge\delta_t} + 1_{f<\delta_t}\,\mathbb{P}\bigl(N_1(t)d^+(\hat\mu_1,\mu_1)>\delta_t-f\mid f\bigr)\Bigr].
\]
When $f<\delta_t$, use peeling arguments with $N_1(t)\in(n_{m-1},n_m]$ with $n_m=\lceil\gamma^m\rceil$, $\gamma=\delta_t/(\delta_t-1)$, $M:= \left\lceil \ln t / \ln \gamma \right\rceil \le\lceil\delta_t\ln t\rceil$. In block $m$, set 
\[
\varepsilon_m:=\sqrt{\frac{2\sigma^2(\delta_t-f)}{n_m}},\quad S_i:=\sum_{j=1}^i X_j.
\]
Using the martingale $W_{\lambda,i}:=\exp(\lambda S_i - i\phi(\lambda))$, $\phi(\lambda)=\lambda\mu_1+\frac{\lambda^2\sigma^2}{2}$, pick $\lambda=-\varepsilon_m/\sigma^2$. By Doob's maximal inequality,
\[
\mathbb{P}(\exists i\in(n_{m-1},n_m]:S_i\le i(\mu_1-\varepsilon_m))
\le \exp\bigl(-(n_{m-1}+1)(\delta_t-f)/n_m\bigr)
\le \exp\bigl(-(\delta_t-f)/\gamma\bigr).
\]
Summing $m=1,\dots,M$ gives 
\[
\mathbb{P}(\mu_1>U_1(t)\mid f)\le 1_{f\ge\delta_t} + M e^{-(\delta_t-f)/\gamma} \leq 1_{f\ge\delta_t} + \lceil\delta_t\ln t\rceil e^{-(\delta_t-f)/\gamma} 
= 1_{f\ge\delta_t} + O\Bigl(\frac{e^f}{t\ln t}\Bigr).
\]
Since $1_{f\ge\delta_t}\le e^{f-\delta_t}=O(e^f/(t\ln^3 t))$, 
\[
\mathbb{P}(\mu_1>U_1(t)) = O\left(\frac{\mathbb{E}[e^{f^+_1(\mu_1)}]}{t\ln t}\right).
\]
\end{proof}

\ap{J'ai rajouté les supplementary materials à la ligne suivante, il suffit de ne pas mettre en commentaire input{supplement.tex} pour l'afficher }

\input{supplement.tex}

\end{document}

%% file: supplement.tex
\section*{Supplementary materials}

\subsection*{Detail proof of Lemma \ref{lem:lemma2}}

\begin{proof}

We want to show that 
\begin{equation*}
    \mathbb{E}\left[ \left( \delta_t - \frac{N'_a}{2 \sigma'^2} \left(\mu_1 - (\hat \mu'_a + L_a) \right)_+^2 \right)_+ \right] \leq \left( \delta_t - \frac{N_a'}{2 \sigma'^2}\left( \mu_1 - (\mu_a' + L_a) \right)_+^2 \right)_+ + O(\sqrt{\delta_t}) 
\end{equation*}

\paragraph{STEP 1 : Rewriting the left-hand side}

\begin{align*}
    \mathbb{E}\left[ \left( \delta_t - \frac{N'_a}{2 \sigma'^2} \left(\mu_1 - (\hat \mu'_a + L_a) \right)_+^2 \right)_+ \right]& = \int_{0}^{+ \infty} \mathbb{P}\left( \left(\delta_t - \frac{N'_a}{2 \sigma'^2} \left(\mu_1 - (\hat \mu'_a + L_a )\right)_+^2\right)_+ > t \right)dt \\
    & = \int_{0}^{+ \infty} \mathbb{P} \left(\delta_t - \frac{N'_a}{2 \sigma'^2} \left(\mu_1 - (\hat \mu'_a + L_a) \right)_+^2 > t \right)dt
\end{align*}
Let $\xi = \mu_1 - L_a - \mu'_a$ and $Z = \mu'_a - \hat \mu'_a \sim \mathcal{N}(0,\sigma'^2 / N'_a)$, and set $\delta := \delta_t$, then we have $\mu_1 - (\hat \mu'_a + L_a ) = \xi + Z$.
\begin{align*}
    \mathbb{E}\left[ \left( \delta_t - \frac{N'_a}{2 \sigma^2} \left(\mu_1 - (\hat \mu'_a + L_a) \right)_+^2 \right)_+ \right] & = \int_{0}^{+ \infty} \mathbb{P} \left(\delta - \frac{N'_a}{2 \sigma'^2} \left(\xi + Z\right)_+^2 > t \right)dt \\
    & = \int_{0}^{+ \infty} \mathbb{P} \left( \frac{N'_a}{2 \sigma'^2} \left(\xi + Z\right)_+^2 \leq \delta - t \right)dt \\
    & =  \int_{0}^{\delta} \mathbb{P} \left( \frac{N'_a}{2 \sigma'^2} \left(\xi + Z\right)_+^2 \leq \delta - t \right)dt \\
    & = \int_{0}^{\delta} \mathbb{P} \left( \sqrt{\frac{N'_a}{2 \sigma'^2}} \left(\xi + Z\right)_+ \leq \sqrt{\delta - t} \right)dt \\
    & =\int_{0}^{\delta} \mathbb{P} \left( \sqrt{\frac{N'_a}{2 \sigma'^2}} \left(\xi + Z\right) \leq \sqrt{\delta - t} \right)dt \\
    & = \int_{0}^{\delta} \mathbb{P} \left( \sqrt{\frac{N'_a}{\sigma'^2}}Z \leq \sqrt{2} \sqrt{\delta - t}  - \sqrt{\frac{N'_a}{ \sigma'^2}} \xi\right)dt \\
    & = \int_{0}^{\delta} \Phi \left(  \sqrt{2} \sqrt{\delta - t}  - \sqrt{\frac{N'_a}{ \sigma'^2}} \xi\right)dt
\end{align*}
Let $\beta := \xi\sqrt{\frac{N'_a}{ \sigma'^2}}$. Then we have
\begin{equation*}
    \mathbb{E}\left[ \left( \delta - \frac{N'_a}{2 \sigma^2} \left(\mu_1 - (\hat \mu'_a + L_a) \right)_+^2 \right)_+ \right] = \int_{0}^{\delta} \Phi \left(  \sqrt{2} \sqrt{\delta - t}  - \beta\right)dt
\end{equation*}
We define 
\begin{equation*}
    I(\beta,\delta) := \int_{0}^{\delta} \Phi \left(  \sqrt{2} \sqrt{\delta - t}  - \beta\right)dt
\end{equation*}

\paragraph{STEP 2: Change of variable \\}

Let us perform the change of variable $u = \sqrt{\delta - t}$
\begin{itemize}
    \item $u = \sqrt{\delta - t}$
    \item $t = \delta - u^2$
    \item $dt = - 2 u\,du$
\end{itemize}
Which gives
\begin{align*}
    I(\beta,\delta) & = \int_{0}^{\delta} \Phi \bigl(  \sqrt{2} \sqrt{\delta - t}  - \beta\bigr)\,dt \\
    & = \int_{\sqrt{\delta}}^{0} \Phi \bigl(  \sqrt{2} u  - \beta\bigr)(-2u\,du) \\
    & = 2\int_{0}^{\sqrt{\delta}} \Phi \bigl(  \sqrt{2} u  - \beta\bigr)u\,du
\end{align*}
We now make the change of variable
\begin{itemize}
    \item $x = \sqrt{2}u - \beta$
    \item $u = \frac{1}{\sqrt{2}}(x+\beta)$
    \item $du = \frac{1}{\sqrt{2}}\,dx$
\end{itemize}
We then obtain
\begin{align*}
    I(\beta,\delta) & = 2\int_{0}^{\sqrt{\delta}} \Phi \bigl(  \sqrt{2} u  - \beta\bigr)u\,du \\
    & = 2\int_{-\beta}^{\sqrt{2\delta}-\beta} \Phi(x)\,\frac{x+\beta}{\sqrt{2}} \,\frac{1}{\sqrt{2}}\,dx \\
    & = \int_{-\beta}^{\sqrt{2\delta}-\beta}(x+\beta)\,\Phi(x)\,dx
\end{align*}
We are thus left with two integrals:
\begin{equation*}
    I(\beta,\delta) = \underbrace{\int_{-\beta}^{\sqrt{2\delta}-\beta} x\Phi(x)\,dx}_{A(\beta,\delta)} + \beta\underbrace{\int_{-\beta}^{\sqrt{2\delta}-\beta} \Phi(x)\,dx}_{B(\beta,\delta)}
\end{equation*}

\paragraph{STEP 3: Computation of $A(\beta,\delta)$ \\}

We begin with an integration by parts:
\begin{itemize}
    \item $u(x) = \Phi(x)$
    \item $u'(x) = \Phi'(x) = \phi(x)$
    \item $v(x) = \frac{x^2}{2}$
    \item $v'(x) = x$
\end{itemize}
Which gives:
\begin{align*}
    A(\beta,\delta) 
    &:= \int_{-\beta}^{\sqrt{2\delta}-\beta} x\Phi(x)\,dx \\
    &=
    \left[ \frac{x^2}{2}\,\Phi(x) \right]_{x=-\beta}^{\,x=\sqrt{2\delta}-\beta}
    -\int_{-\beta}^{\sqrt{2\delta}-\beta} \frac{x^2}{2}\phi(x)\,dx \\
    &=
    \frac{\bigl(\sqrt{2\delta}-\beta\bigr)^2}{2}\,\Phi\bigl(\sqrt{2\delta}-\beta\bigr)
    -\frac{\beta^2}{2}\,\Phi(-\beta)
    -\int_{-\beta}^{\sqrt{2\delta}-\beta} \frac{x^2}{2}\phi(x)\,dx
\end{align*}
Let us now compute $ \int_{-\beta}^{\sqrt{2\delta}-\beta} x^2\phi(x)\,dx$. 
Since $\phi'(x) = -x\phi(x)$, we have
\begin{equation*}
    \int_{-\beta}^{\sqrt{2\delta}-\beta} x^2\phi(x)\,dx 
= -\int_{-\beta}^{\sqrt{2\delta}-\beta} x\phi'(x)\,dx
\end{equation*}
We apply integration by parts:
\begin{itemize}
    \item $u(x) = x$
    \item $u'(x) = 1$
    \item $v(x) = \phi(x)$
    \item $v'(x) = \phi'(x)$
\end{itemize}
Therefore,
\begin{align*}
    \int_{-\beta}^{\sqrt{2\delta}-\beta} x\phi'(x)\,dx 
    &= \Bigl[x\phi(x)\Bigr]_{x=-\beta}^{x=\sqrt{2\delta}-\beta}
      - \int_{-\beta}^{\sqrt{2\delta}-\beta} \phi(x)\,dx \\
    &= \bigl(\sqrt{2\delta}-\beta\bigr)\phi\bigl(\sqrt{2\delta}-\beta\bigr)
       - \bigl(-\beta\,\phi(-\beta)\bigr)
       - \bigl[\Phi\bigl(\sqrt{2\delta}-\beta\bigr) - \Phi(-\beta)\bigr] \\
    &= \bigl(\sqrt{2\delta}-\beta\bigr)\phi\bigl(\sqrt{2\delta}-\beta\bigr)
       + \beta\phi(-\beta)
       - \Phi\bigl(\sqrt{2\delta}-\beta\bigr)
       + \Phi(-\beta).
\end{align*}
It follows that
\begin{equation*}
    \int_{-\beta}^{\sqrt{2\delta}-\beta} x^2\phi(x)\,dx 
= -\Bigl[
(\sqrt{2\delta}-\beta)\phi(\sqrt{2\delta}-\beta)
+ \beta\phi(-\beta)
- \Phi(\sqrt{2\delta}-\beta)
+ \Phi(-\beta)
\Bigr]
\end{equation*}
Plugging this into $A(\beta,\delta)$, we get:
\begin{align*}
    A(\beta,\delta)
    &= \frac{\bigl(\sqrt{2\delta}-\beta\bigr)^2}{2}\,\Phi\bigl(\sqrt{2\delta}-\beta\bigr)
      - \frac{\beta^2}{2}\,\Phi(-\beta) \\
    &\quad 
      - \frac{1}{2}\bigl[-(\sqrt{2\delta}-\beta)\phi(\sqrt{2\delta}-\beta)
      -\beta\phi(-\beta)
      +\Phi(\sqrt{2\delta}-\beta)
      -\Phi(-\beta)\bigr] \\
    &= \frac{2\delta + \beta^2 - 2\beta\sqrt{2\delta}}{2}\,\Phi\bigl(\sqrt{2\delta}-\beta\bigr)
      - \frac{\beta^2}{2}\,\Phi(-\beta) \\
    &\quad 
      + \frac{(\sqrt{2\delta}-\beta)\phi(\sqrt{2\delta}-\beta) + \beta\phi(-\beta) 
      - \Phi(\sqrt{2\delta}-\beta) + \Phi(-\beta)}{2}.
\end{align*}
Which yields:
\begin{equation*}
    A(\beta,\delta) 
= \frac{\bigl(\sqrt{2\delta}-\beta\bigr)^2}{2}\,\Phi\bigl(\sqrt{2\delta}-\beta\bigr)
- \frac{\beta^2}{2}\,\Phi(-\beta)
+ \frac{1}{2}\Bigl[
(\sqrt{2\delta}-\beta)\phi(\sqrt{2\delta}-\beta)
+ \beta\phi(-\beta)
- \Phi(\sqrt{2\delta}-\beta)
+ \Phi(-\beta)
\Bigr]
\end{equation*}

\paragraph{STEP 4: Computation of $B(\beta,\delta)$\\}

We perform integration by parts:
\begin{itemize}
    \item $u(x) = \Phi(x)$
    \item $u'(x) = \phi(x)$
    \item $v(x) = x$
    \item $v'(x) = 1$
\end{itemize}
Which gives:
\begin{align*}
    B(\beta,\delta) & := \int_{-\beta}^{\sqrt{2\delta}-\beta} \Phi(x)\,dx \\
    & = \left[ x \,\Phi(x)\right]_{-\beta}^{\sqrt{2\delta}-\beta} - \int_{-\beta}^{\sqrt{2\delta}-\beta}x \,\phi(x)\,dx \\
    & = \left[ x \,\Phi(x)\right]_{-\beta}^{\sqrt{2\delta}-\beta} + \left[ \phi(x)\right]_{-\beta}^{\sqrt{2\delta}-\beta} \\
    & = (\sqrt{2\delta}-\beta)\Phi(\sqrt{2\delta}-\beta) + \beta \Phi(-\beta) + \phi(\sqrt{2\delta}-\beta) - \phi(\beta)
\end{align*}

\paragraph{STEP 5: Computation of $I(\beta,\delta)$ \\}

We have $I(\beta,\delta) = A(\beta,\delta) + \beta \,B(\beta,\delta)$. Recall the formulas:
\begin{equation*}
    A(\beta,\delta) = \frac{\bigl(\sqrt{2\delta}-\beta\bigr)^2}{2}\,\Phi\bigl(\sqrt{2\delta}-\beta\bigr)
- \frac{\beta^2}{2}\,\Phi(-\beta)
+ \frac{1}{2}\Bigl[
(\sqrt{2\delta}-\beta)\phi(\sqrt{2\delta}-\beta)
+ \beta\phi(-\beta)
- \Phi(\sqrt{2\delta}-\beta)
+ \Phi(-\beta)
\Bigr]
\end{equation*}
And
\begin{equation*}
    B(\beta,\delta) = (\sqrt{2\delta}-\beta)\Phi(\sqrt{2\delta}-\beta) + \beta \Phi(-\beta) + \phi(\sqrt{2\delta}-\beta) - \phi(\beta)
\end{equation*}

Let’s expand term by term:
\\
\\
1. \textbf{Coefficient of $\Phi(\sqrt{2\delta}-\beta)$}
\\
\\
- In $A(\beta,\delta)$, the coefficient is $\frac{\bigl(\sqrt{2\delta}-\beta\bigr)^2}{2} - \frac{1}{2} = \frac{2\delta + \beta^2 - 2\beta\sqrt{2\delta} - 1}{2}$
\\
\\
- In $\beta\,B(\beta,\delta)$, the coefficient is $\beta(\sqrt{2\delta}-\beta) = \frac{2\beta\sqrt{2\delta} - 2 \beta^2}{2}$
\\
\\
- Summing: $\frac{2\delta + \beta^2 - 2\beta\sqrt{2\delta} - 1}{2} +\frac{2\beta\sqrt{2\delta} - 2 \beta^2}{2} = \frac{2 \delta - \beta^2 - 1}{2}$
\\
\\
2. \textbf{Coefficient of $\Phi(-\beta)$}
\\
\\
- In $A(\beta,\delta)$, the coefficient is $-\frac{\beta^2}{2} + \frac{1}{2} = -\frac{\beta^2 - 1}{2}$
\\
\\
- In $\beta\,B(\beta,\delta)$, the coefficient is $\frac{2\beta^2}{2}$
\\
\\
- Summing: $-\frac{\beta^2 - 1}{2} + \frac{2\beta^2}{2} = \frac{\beta^2 + 1}{2}$
\\
\\
3. \textbf{Coefficient of $\varphi(\sqrt{2\delta}-\beta)$}
\\
\\
- In $A(\beta,\delta)$, the coefficient is $\frac{\sqrt{2\delta} -\beta}{2}$
\\
\\
- In $\beta\,B(\beta,\delta)$, the coefficient is $\frac{2\beta}{2}$
\\
\\
- Summing: $\frac{\sqrt{2\delta} -\beta}{2} + \frac{2\beta}{2} = \frac{\sqrt{2\delta} +\beta}{2}$
\\
\\
4. \textbf{Coefficient of $\phi(\beta)$}
\\
\\
- In $A(\beta,\delta)$, the coefficient is $\frac{\beta}{2}$
\\
\\
- In $\beta\,B(\beta,\delta)$, the coefficient is $-\frac{2\beta}{2}$
\\
\\
- Summing: $\frac{\beta}{2}-\frac{2\beta}{2} = -\frac{\beta}{2}$
\\
\\
Finally, we obtain:
\begin{equation*}
    I(\beta,\delta) = \frac{2 \delta - \beta^2 - 1}{2}\,\Phi(\sqrt{2\delta}-\beta) + \frac{\beta^2 + 1}{2}\,\Phi(-\beta) + \frac{\sqrt{2 \delta} +\beta}{2}\,\varphi(\sqrt{2\delta}-\beta) -\frac{\beta}{2}\,\varphi(\beta) 
\end{equation*}

\paragraph{STEP 6: Comparison with $\left( \delta - \frac{\beta^2_+}{2} \right)_+$ \\}

1. \textbf{For $\beta \leq 0$}
\begin{equation*}
    \left( \delta - \frac{\beta^2_+}{2} \right)_+ = \delta
\end{equation*}
And
\begin{equation*}
    I(\beta,\delta) := \int_{0}^{\delta} \Phi \bigl(  \sqrt{2} \sqrt{\delta - t}  - \beta\bigr)\,dt \leq \delta
\end{equation*}
Therefore, for $\beta \leq 0$
\begin{equation*}
    I(\beta,\delta) \leq \left( \delta - \frac{\beta^2_+}{2} \right)_+
\end{equation*}
\\
2. \textbf{For $\sqrt{2 \delta} \leq \beta $}
\begin{equation*}
    \left( \delta - \frac{\beta^2_+}{2} \right)_+ = 0
\end{equation*}
And using Mill's ration inequality $\Phi(-x) \leq \frac{\phi(x)}{x}$
\begin{align*}
    I(\beta,\delta) & = \underbrace{\frac{2 \delta - \beta^2 - 1}{2}\Phi(\sqrt{2\delta}-\beta)}_{\leq 0} + \frac{\beta^2 + 1}{2}\Phi(-\beta) + \underbrace{\frac{\sqrt{2 \delta} +\beta}{2}}_{\leq \beta}\varphi(\sqrt{2\delta}-\beta) \underbrace{-\frac{\beta}{2}\varphi(\beta)}_{\leq 0} \\
    & \leq \frac{\beta^2 + 1}{2}\Phi(-\beta) + \beta\,\varphi(\sqrt{2\delta}-\beta) \\
    & \leq \frac{\beta^2 + 1}{2} \frac{\phi(\beta)}{\beta} + \left(\sqrt{2\delta} + (\beta - \sqrt{2\delta} ) \right) \phi \left(\beta - \sqrt{2\delta} \right) \\
    & \leq \beta \phi(\beta) + \sqrt{2 \delta} + (\beta - \sqrt{2\delta} )\phi \left(\beta - \sqrt{2\delta} \right)
\end{align*}
Using that for all $a >0$, $a\phi(a) \leq 1 \phi(1)$ we have that 
\begin{equation*}
    I(\beta,\delta) = O(\sqrt{\delta})
\end{equation*}
So for $\sqrt{2 \delta} \leq \beta $
\begin{equation*}
    I(\beta,\delta) \leq \left( \delta - \frac{\beta^2_+}{2} \right)_+ + O(\sqrt{\delta})
\end{equation*}
\\
3. \textbf{For $0 \leq \beta \leq \sqrt{2 \delta} $}
\begin{equation*}
    \left( \delta - \frac{\beta^2_+}{2} \right)_+ = \left(\delta - \frac{\beta^2}{2} \right)
\end{equation*}
And
\begin{align*}
    I(\beta,\delta) & = \underbrace{\frac{2 \delta - \beta^2 - 1}{2}\Phi(\sqrt{2\delta}-\beta)}_{\leq \left( \delta - \frac{\beta^2}{2} \right)} + \frac{\beta^2 + 1}{2}\Phi(-\beta) +\underbrace{\sqrt{\frac{\delta}{\pi}}}_{\text{from }\frac{\sqrt{2 \delta} +\beta}{2}\varphi(\sqrt{2\delta}-\beta)} \underbrace{-\frac{\beta}{2}\varphi(\beta)}_{\leq 0} \\
    & \leq \left( \delta - \frac{\beta^2}{2} \right) + \frac{\beta^2 + 1}{2}\Phi(-\beta) +\sqrt{\frac{\delta}{\pi}} 
\end{align*}
Now if $\beta \in [0,C]$ with $C$ a constant then $\frac{\beta^2 + 1}{2}\Phi(-\beta)$ is bounded so 
\begin{equation*}
    I(\beta,\delta) =  \left( \delta - \frac{\beta^2}{2} \right) + \frac{\beta^2 + 1}{2}\Phi(-\beta) +\sqrt{\frac{\delta}{\pi}} = \left( \delta - \frac{\beta^2}{2} \right) + O(\sqrt{\delta})
\end{equation*}
If $\beta \underset{\beta \to \infty}{\longrightarrow} +\infty$ but slower than $\sqrt{2 \delta}$, then by using Mill's ration $\Phi(-\beta) \leq \frac{\phi(\beta)}{\beta}$, we have
\begin{align*}
    \frac{\beta^2 + 1}{2}\Phi(-\beta) &\leq\frac{\beta^2 + 1}{2} \cdot \frac{\phi(\beta)}{\beta} \\
    & = \frac{\beta^2 + 1}{2 \sqrt{2 \pi}} \cdot \frac{e^{-\beta^2 / 2}}{\beta} \\
    & = O(\sqrt{\beta}) \\
\end{align*}
So for $\beta$ big enough we have a constant $C'$ such that
\begin{equation*}
     \frac{\beta^2 + 1}{2}\Phi(-\beta) \leq C' \sqrt{\beta} \leq C'\sqrt{\sqrt{2 \delta}} = o(\sqrt{\delta})
\end{equation*}
Therefore, for $0 \leq \beta \leq \sqrt{2 \delta} $,
\begin{equation*}
    I(\beta,\delta) \leq  \left( \delta - \frac{\beta^2_+}{2} \right)_+ + O(\sqrt{\delta})
\end{equation*}
In all cases,
\begin{equation*}
    I(\beta,\delta) \leq \left( \delta - \frac{\beta^2_+}{2} \right)_+ + O(\sqrt{\delta})
\end{equation*}

\paragraph{Conclusion}

We have:
\begin{equation*}
    \mathbb{E}\left[ \left( \delta_t - \frac{N'_a}{2 \sigma'^2} \left(\mu_1 - (\hat \mu'_a + L_a) \right)_+^2 \right)_+ \right] \leq \left( \delta_t - \frac{N_a'}{2 \sigma'^2}\left( \mu_1 - (\mu_a' + L_a) \right)_+^2 \right)_+ + O(\sqrt{\delta_t})
\end{equation*}
\end{proof}

\subsection*{Detail proof of Lemma \ref{lem:lemma3}}

\begin{proof}
If $N'_1 = 0$ it's immediate. We will consider the case where $N'_1 \geq 1$. 
    We have $\eta := L_1 + \mu'_1 - \mu_1 > 0$, and then:
\begin{align*}
    f^+_1(\mu_1)
    &= N'_1 \frac{\bigl(\mu_1 - (\hat\mu'_1 + L_1)\bigr)_+^2}{2\,\sigma'^2}
     \\
    & = N'_1 \frac{\bigl((\mu'_1 - \hat\mu'_1) - \eta \bigr)_+^2}{2\,\sigma'^2}.
\end{align*}
Let $Z := \mu'_1 - \hat\mu'_1$. Then $Z \sim \mathcal{N}\bigl(0,\tfrac{\sigma'^2}{N'_1}\bigr)$. Let us define:
\begin{equation*}
    \tau := \frac{\sigma'}{\sqrt{N'_1}}, 
  \quad 
  W := \frac{Z}{\tau} \sim \mathcal{N}(0,1).
\end{equation*}

Therefore,
\begin{equation*}
    Z - \eta = \tau\,W - \eta,
  \quad
  (Z - \eta)_+ = \tau\,\bigl(W - \tfrac{\eta}{\tau}\bigr)_+.
\end{equation*}

Thus,
\begin{equation*}
    f^+_1(\mu_1)
  = 
  \begin{cases}
    0, & \text{if } Z \le \eta \quad (\text{i.e., } W \le \tfrac{\eta}{\tau}), \\[1ex]
    N'_1\,\frac{(Z - \eta)^2}{2\,\sigma'^2}, & \text{if } Z > \eta \quad (\text{i.e., } W > \tfrac{\eta}{\tau}).
  \end{cases}
\end{equation*}

For $Z>\eta$,
\begin{equation*}
    N'_1\,\frac{(Z - \eta)^2}{2\,\sigma'^2}
  = \tfrac12\,\bigl(W - \tfrac{\eta}{\tau}\bigr)^2.
\end{equation*}
Thus:
\begin{equation*}
    f^+_1(\mu_1)
  =
  \begin{cases}
    0, & \text{if } W \le a, \\[1ex]
    \displaystyle \tfrac12\,\bigl(W - a\bigr)^2, & \text{if } W > a,
  \end{cases}
  \quad
  \text{where} \ a := \frac{\eta}{\tau} = \frac{\eta\,\sqrt{N'_1}}{\sigma'}.
\end{equation*}
We now aim to upper-bound:
\begin{equation*}
    \mathbb{E}\bigl[e^{\,f^+_1(\mu_1)}\bigr]
  = \int_{-\infty}^{+\infty} \exp\bigl(f^+_1(\mu_1)\bigr)\,\varphi(w)\,dw,
  \quad
  \varphi(w) = \frac{1}{\sqrt{2\pi}}\,e^{-\,w^2/2}.
\end{equation*}
Since:
\begin{equation*}
    \exp\bigl(f^+_1(\mu_1)\bigr)
  =
  \begin{cases}
    1, & \text{if } W \le a, \\[1ex]
    \exp\Bigl(\tfrac12\,(W - a)^2\Bigr), & \text{if } W > a,
  \end{cases}
\end{equation*}
we split the integral:
\[
  \begin{aligned}
    \mathbb{E}\bigl[e^{\,f^+_1(\mu_1)}\bigr]
    &= \int_{-\infty}^a 1 \cdot \varphi(w)\,dw 
      + \int_a^{+\infty} \exp\Bigl(\tfrac12\,(w - a)^2\Bigr)\,\varphi(w)\,dw \\[1ex]
    &= \underbrace{\int_{-\infty}^a \varphi(w)\,dw}_{\Phi(a)} 
      + \underbrace{\int_a^{+\infty} \exp\Bigl(\tfrac12\,(w - a)^2\Bigr)\,\varphi(w)\,dw}_{J(a)}.
  \end{aligned}
\]
For $w>a$, we get:
\begin{equation*}
    \varphi(w)\,\exp\Bigl(\tfrac12\,(w - a)^2\Bigr)
  = \frac{1}{\sqrt{2\pi}}\,\exp\Bigl(-\,a\,w + \tfrac{a^2}{2}\Bigr).
\end{equation*}
Thus:
\begin{align*}
    J(a)
  &= \int_a^{+\infty} \frac{1}{\sqrt{2\pi}}\,\exp\Bigl(-\,a\,w + \tfrac{a^2}{2}\Bigr)\,dw \\
  & = e^{\,a^2/2}\,\frac{1}{\sqrt{2\pi}} \int_a^{+\infty} e^{-\,a\,w}\,dw \\
  & = e^{\,a^2/2}\,\frac{1}{\sqrt{2\pi}} \cdot \frac{1}{a}\,e^{-\,a^2} \\
  & = \frac{1}{\sqrt{2\pi}}\,\frac{1}{a}\,e^{-\,\tfrac{a^2}{2}}.
\end{align*}

Finally, with $a := \frac{\eta\,\sqrt{N'_a}}{\sigma'}$,
\begin{equation*}
    \mathbb{E}\bigl[e^{\,f^+_1(\mu_1)}\bigr]
    = \Phi(a) + \frac{1}{\sqrt{2\pi}}\,\frac{1}{a}\,e^{-\,\tfrac{a^2}{2}}
    = \Phi(a) + \frac{\varphi(a)}{a}
\end{equation*}
For all $N'_a \ge 1$,
\begin{align*}
    a = \frac{\eta\,\sqrt{N'_1}}{\sigma'} \ge \frac{\eta}{\sigma'},
\end{align*}
Then we define
\begin{equation*}
    C' := 1+\frac{\sigma'}{\eta\sqrt{2 \pi}} \geq \Phi(a) + \frac{\varphi(a)}{a}.
\end{equation*}

\end{proof}

\subsection*{Detail proof of Lemma \ref{lem:lemma4}}

\begin{proof}
We aim to bound 
\begin{equation*}
  \mathbb{P}\bigl(\mu_1 > U_1(t)\bigr).
\end{equation*}
By definition of $U_1(t)$, conditioning on the random term $f^+_1(\mu_1)$, we write
\begin{equation*}
  \mathbb{P}\bigl(\mu_1 > U_1(t)\bigr)
  = \mathbb{P}\bigl(N_1(t)\,d^+(\hat\mu_1(t),\mu_1) + f^+_1(\mu_1) > \delta_t\bigr)
  = \mathbb{E}\Bigl[
      \mathbb{P}\bigl(N_1(t)\,d^+(\hat\mu_1(t),\mu_1) > \delta_t - f^+_1(\mu_1)\,\bigm|\,f^+_1\bigl)\Bigr].
\end{equation*}
Fix an outcome of $f^+_1$, and assume $\delta_t > f^+_1(\mu_1)$ (otherwise the probability is trivially $\le 1$ and the same bound follows). Under $\delta_t > f^+_1(\mu_1)$, the event
\begin{equation*}
  N_1(t)\,d^+(\hat\mu_1(t),\mu_1) > \delta_t - f^+_1(\mu_1)
\end{equation*}
is equivalent (for Gaussian rewards, variance $\sigma^2$) to
\begin{equation*}
  \hat\mu_1(t) \le \mu_1 - \sqrt{\frac{2\sigma^2\bigl(\delta_t - f^+_1(\mu_1)\bigr)}{N_1(t)}}.
\end{equation*}
We control
\begin{equation*}
  \mathbb{P}\Bigl(\hat\mu_1(t) \le \mu_1 - \sqrt{\tfrac{2\sigma^2(\delta_t - f^+_1(\mu_1))}{N_1(t)}} \mid f^+_1\Bigr)
\end{equation*}
by a standard peeling (time-index partition) argument.

\paragraph{Peeling over $N_1(t)$.}
Let $\gamma>1$. Define 
\begin{equation*}
  M = \left\lceil \frac{\ln t}{\ln \gamma}\right\rceil,\quad
  n_m = \lceil \gamma^m \rceil \quad(m=1,2,\dots,M),\quad n_0=0.
\end{equation*}
Then $N_1(t)$ lies in some interval $\{\,n_{m-1}+1,\dots,n_m\}$. For each block $m=1,\dots,M$, set
\begin{equation*}
  \varepsilon_m \;:=\; \sqrt{\frac{2\sigma^2\bigl(\delta_t - f^+_1(\mu_1)\bigr)}{\,n_m\,}}.
\end{equation*}
For all $i\in[n_{m-1}+1,n_m]$:
\begin{equation*}
  \sqrt{\frac{2\sigma^2(\delta_t - f^+_1(\mu_1))}{i}}
  \;\ge\; \sqrt{\frac{2\sigma^2(\delta_t - f^+_1(\mu_1))}{n_m}}
  = \varepsilon_m.
\end{equation*}
Hence for all $i\in[n_{m-1}+1,n_m]$
\begin{equation*}
  \bigl\{\hat\mu_{1}(i) \le \mu_1 - \sqrt{\tfrac{2\sigma^2(\delta_t - f^+_1(\mu_1))}{\,i\,}}\bigr\}
  \;\subseteq\;
  \bigl\{\hat\mu_{1}(i) \le \mu_1 - \varepsilon_m\bigr\}.
\end{equation*}
Thus by union bound,
\begin{equation*}
  \mathbb{P}\Bigl(\hat\mu_1(t) \le \mu_1 - \sqrt{\tfrac{2\sigma^2(\delta_t - f^+_1(\mu_1))}{N_1(t)}} \mid f^+_1\Bigr)
  \;\le\;
  \sum_{m=1}^M 
  \mathbb{P}\Bigl(\exists\,i\in[n_{m-1}+1,n_m]: \hat\mu_{1}(i) \le \mu_1 - \varepsilon_m \;\Big|\; f^+_1\Bigr).
\end{equation*}
Writing $\hat\mu_{1}(i) = \tfrac1i\sum_{j=1}^i X_j$, with $X_j\sim\mathcal{N}(\mu_1,\sigma^2)$ i.i.d., the event $\hat\mu_{1}(i)\le \mu_1 - \varepsilon_m$ is equivalent to
\begin{equation*}
  \sum_{j=1}^i X_j \le i(\mu_1 - \varepsilon_m).
\end{equation*}
We now control
\begin{equation*}
  \mathbb{P}\Bigl(\exists\,i\in[n_{m-1}+1,n_m]: S_i \le i(\mu_1 - \varepsilon_m)\Bigr),
  \quad S_i := \sum_{j=1}^i X_j.
\end{equation*}

\paragraph{Exponential martingale \& Doob’s inequality.}
Define the log-MGF of one reward:
\begin{equation*}
  \phi(\lambda) := \ln \mathbb{E}[e^{\lambda X_1}] = \lambda \mu_1 + \tfrac{\lambda^2\sigma^2}{2}, 
  \quad \lambda\in\mathbb{R}.
\end{equation*}
For any fixed $\lambda<0$, consider
\begin{equation*}
  W_{\lambda,i} := \exp\bigl(\lambda S_i - i\,\phi(\lambda)\bigr).
\end{equation*}
Since $\mathbb{E}[e^{\lambda X_j}] = e^{\phi(\lambda)}$, $(W_{\lambda,i})_{i\ge0}$ is a positive martingale. By Doob’s maximal inequality, for all $\lambda < 0$
\begin{equation}\label{eq:peel1}
  \mathbb{P}\Bigl(\max_{i\in[n_{m-1}+1,n_m]} W_{\lambda,i} \ge u\Bigr)
  \;\le\; \frac{\mathbb{E}[W_{\lambda,n_m}]}{u}
  = \frac{1}{u},
  \quad u>0.
\end{equation}
On the other hand, for each $i\ge n_{m-1}+1$,
\begin{equation*}
  \{\;S_i \le i(\mu_1 - \varepsilon_m)\}
  \;\Longrightarrow\;
  W_{\lambda,i} \;=\; \exp\bigl(\lambda S_i - i\phi(\lambda)\bigr)
  \;\ge\; \exp\bigl(\lambda\,i(\mu_1 - \varepsilon_m) - i\phi(\lambda)\bigr).
\end{equation*}
Thus
\begin{equation}\label{eq:peel2}
  \mathbb{P}\Bigl(\exists i\in[n_{m-1}+1,n_m]: S_i \le i(\mu_1 - \varepsilon_m)\Bigr)
  \;\le\;
  \mathbb{P}\Bigl(\max_{i\in[n_{m-1}+1,n_m]} W_{\lambda,i}
  \;\ge\; \exp\bigl(\lambda\,i(\mu_1 - \varepsilon_m) - i\phi(\lambda)\bigr)\Bigr).
\end{equation}
Since $\lambda<0$ and $\phi(\lambda) = \lambda\mu_1 + \tfrac{\lambda^2\sigma^2}{2}$,
\begin{equation*}
  \lambda\,i(\mu_1 - \varepsilon_m) - i \phi(\lambda)
  = -\,i\Bigl(\lambda \varepsilon_m + \tfrac{\lambda^2\sigma^2}{2}\Bigr).
\end{equation*}
For all $i\ge n_{m-1}+1$, this is bounded by
\begin{equation*}
  - (n_{m-1}+1)\Bigl(\lambda \varepsilon_m + \tfrac{\lambda^2\sigma^2}{2}\Bigr).
\end{equation*}
Hence, by Equations \eqref{eq:peel1} and \eqref{eq:peel2}
\begin{equation*}
  \mathbb{P}\Bigl(\exists i\in[n_{m-1}+1,n_m]: S_i \le i(\mu_1 - \varepsilon_m)\Bigr)
  \;\le\; \exp\Bigl((n_{m-1}+1)\bigl(\lambda \varepsilon_m + \tfrac{\lambda^2\sigma^2}{2}\bigr)\Bigr).
\end{equation*}

\paragraph{Choice of $\lambda$.}
To optimize the exponent, set $\lambda = -\varepsilon_m / \sigma^2 < 0$. Then
\begin{equation*}
  \lambda \varepsilon_m + \tfrac{\lambda^2\sigma^2}{2}
  = -\frac{\varepsilon_m^2}{2\sigma^2}
  = -\frac{1}{2\sigma^2}\cdot \frac{2\sigma^2(\delta_t - f^+_1(\mu_1))}{n_m}
  = -\frac{\delta_t - f^+_1(\mu_1)}{n_m}.
\end{equation*}
Thus
\begin{equation*}
  (n_{m-1}+1)\Bigl(\lambda \varepsilon_m + \tfrac{\lambda^2\sigma^2}{2}\Bigr)
  = - (n_{m-1}+1)\,\frac{\delta_t - f^+_1(\mu_1)}{n_m}
  \;\le\; -\frac{1}{\gamma}\bigl(\delta_t - f^+_1(\mu_1)\bigr),
\end{equation*}
since $(n_{m-1}+1)/n_m \ge 1/\gamma$. Therefore
\begin{equation*}
  \mathbb{P}\Bigl(\exists i\in[n_{m-1}+1,n_m]: S_i \le i(\mu_1 - \varepsilon_m)\Bigr)
  \;\le\; \exp\Bigl(-\tfrac{1}{\gamma}(\delta_t - f^+_1(\mu_1))\Bigr).
\end{equation*}
Summing over $m=1,\dots,M$:
\begin{equation*}
  \mathbb{P}\Bigl(\hat\mu_1(t) \le \mu_1 - \sqrt{\tfrac{2\sigma^2(\delta_t - f^+_1(\mu_1))}{N_1(t)}} \mid f^+_1\Bigr)
  \;\le\; \sum_{m=1}^M \exp\Bigl(-\tfrac{1}{\gamma}(\delta_t - f^+_1(\mu_1))\Bigr)
  = M \,\exp\Bigl(-\tfrac{1}{\gamma}(\delta_t - f^+_1(\mu_1))\Bigr).
\end{equation*}

\paragraph{Choice of $\gamma$ and final bound.}
Take $\delta_t = \ln t + 3\ln\ln t$. To simplify the prefactor $M = \lceil \ln t / \ln\gamma\rceil$, choose 
\[
\gamma := \frac{\delta_t}{\delta_t - 1},
\]
so that $\delta_t / \gamma = \delta_t - 1$ and $\ln\gamma = \ln \left( 1 + \frac{1}{\delta_t - 1} \right) \ge \frac{\frac{1}{\delta_t -1}}{\frac{1}{\delta_t -1} + 1} =1/\delta_t$. Then
\begin{equation*}
  M = \Bigl\lceil \frac{\ln t}{\ln\gamma}\Bigr\rceil \;\le\; \lceil \delta_t \ln t\rceil,
  \quad
  \exp\Bigl(-\tfrac{1}{\gamma}(\delta_t - f^+_1(\mu_1))\Bigr)
  = \exp\bigl(-( \delta_t - 1)\bigr)\,\exp\bigl(f^+_1(\mu_1)/\gamma\bigr).
\end{equation*}
Hence
\begin{equation*}
  M \exp\Bigl(-\tfrac{1}{\gamma}(\delta_t - f^+_1(\mu_1))\Bigr)
  \;\le\; \lceil \delta_t \ln t\rceil \,\exp\bigl(-( \delta_t - 1)\bigr)\,\exp\bigl(f^+_1(\mu_1)\bigr).
\end{equation*}
Since $\delta_t = \ln t + 3\ln\ln t$, $\lceil \delta_t \ln t\rceil \exp(-(\delta_t - 1)) = O(1/(t\ln t))$. Thus
\begin{equation*}
  \mathbb{P}\Bigl(N_1(t)\,d^+(\hat\mu_1(t),\mu_1) > \delta_t - f^+_1(\mu_1)\mid f^+_1\Bigr)
  \;\le\; O\bigl(\tfrac{1}{t\ln t}\bigr)\,\exp\bigl(f^+_1(\mu_1)\bigr).
\end{equation*}
Taking expectation over $f^+_1$ 
\begin{equation*}
  \mathbb{P}(\mu_1>U_1(t)) = O\left(\frac{\mathbb{E}[e^{f^+_1(\mu_1)}]}{t\ln t}\right).
\end{equation*}
This completes the proof.
\end{proof}

\subsection*{Explicit formula for $U_a(t)$ in the Gaussian case}

We want an explicit formula for
\begin{equation*}
    U_a(t) := \max \left\{ q \ \middle| \ N_a(t)\frac{( q - \hat \mu_a(t))_+^2}{2 \sigma^2} + N'_a \frac{\left( q - L_a - \hat \mu'_a \right)^2_+}{2 \sigma'^2} \leq \delta_t \right\}
\end{equation*}

Let us denote $\alpha := \frac{N_a(t)}{2 \sigma^2}$, $\beta := \frac{N'_a}{2 \sigma'^2}$, $\mu := \hat \mu_a(t)$, $\mu' := \mu'_a$, $L := L_a$ and $\delta := \delta_t$.
Then,
\begin{equation*}
    U_a(t) = \max \left\{ q \ \middle| \ \alpha ( q - \mu)_+^2 + \beta( q-( \mu'+L))_+^2 \leq \delta \right\}
\end{equation*}
\begin{itemize}
    \item If $\mu \geq  \mu' + L +  \sqrt{\frac{\delta}{\beta}}$ then,
    \begin{equation*}
        U_a(t) =  \mu' + L + \sqrt{\frac{\delta}{\beta}}
    \end{equation*}
    \item If $\mu' + L \geq \mu + \sqrt{\frac{\delta}{\alpha}}$ then,
    \begin{equation*}
        U_a(t) = \mu + \sqrt{\frac{\delta}{\alpha}}
    \end{equation*}
    \item Otherwise,
    \begin{equation*}
        U_a(t) = \frac{\alpha \mu + \beta ( \mu' + L) + \sqrt{ (\alpha + \beta) \delta -  \alpha\beta(\mu - \big( \mu' + L)\big)^2 } }{\alpha + \beta}
    \end{equation*}
\end{itemize}

One simply checks the three cases below:
\begin{itemize}
    \item $\mu' + L \leq q \leq \mu$
    \item $\mu \leq q \leq \mu' + L$
    \item $\max(\mu' + L, \mu) \leq q$
\end{itemize}